\documentclass[runningheads, final]{llncs}
\usepackage[utf8]{inputenc}
\usepackage[T1]{fontenc}
\usepackage{amsmath,amssymb,mathtools,bm}
\usepackage{enumitem,booktabs,graphicx}
\graphicspath{{./}{../}}
\usepackage{algorithm}
\usepackage[noend]{algpseudocode}
\usepackage{hyperref}
\usepackage[capitalize]{cleveref}
\usepackage{thmtools,thm-restate}
\usepackage{todonotes}
\usepackage{tabularx}

\makeatletter
\@ifundefined{endassumption}{\spnewtheorem{assumption}{Assumption}{\bfseries}{\itshape}}{}
\makeatother

\makeatletter
\providecommand{\theHALG@line}{}\renewcommand{\theHALG@line}{\thealgorithm.\arabic{ALG@line}}
\newcommand{\LComment}[1]{\Statex \(\triangleright\) \textit{#1}}
\makeatother

\newcommand{\E}{\mathbb{E}}\newcommand{\KL}{D_{\mathrm{KL}}}
\newcommand{\cS}{\mathcal{S}}\newcommand{\cA}{\mathcal{A}}\newcommand{\cO}{\mathcal{O}}\newcommand{\cD}{\mathcal{D}}
\newcommand{\cF}{\mathcal{F}}\newcommand{\cG}{\mathcal{G}}\newcommand{\cH}{\mathcal{H}}

\newcommand{\cX}{\mathcal{X}}
\newcommand{\cY}{\mathcal{Y}}
\newcommand{\cZ}{\mathcal{Z}}\newcommand{\cT}{\mathcal{T}}\newcommand{\epref}{\tilde{p}}

\title{
Active Inference as
a \\ Convex Markov Decision Process
}
\titlerunning{Active Inference as a Convex MDP}
\author{Nikola Milosevic\inst{1}%
\and
Nicol\'{a}s Hinrichs\inst{1}%
\and
Nico Scherf\inst{1}%
}
\authorrunning{N. Milosevic et al.}
\institute{Max Planck Institute for Human Cognitive and Brain Sciences, Stephanstraße. 1A, 04103 Leipzig, Germany
\email{\{nmilosevic,nscherf,hinrichsn\}@cbs.mpg.de}
}

\begin{document}
\maketitle

\begin{abstract}
Active Inference (AIF) frames adaptive behavior as the minimization of expected free energy (EFE), combining epistemic and pragmatic objectives within a single variational principle. We frame AIF as policy optimization and show that, for closed-loop control policies, EFE minimization can be formulated as a convex Markov decision process (MDP). 
This perspective reveals that policy-dependent reward prediction errors transmit natural gradients of the expected free energy backwards in time rather than up a hierarchy.
Finally, we show that coupling world-model learning with policy optimization gives active inference the structure of performative reinforcement learning.
Together this places EFE minimization within modern reinforcement learning and optimization theory and opens a route toward principled
algorithms for active inference.
\end{abstract}

\section{Introduction}\label{sec:intro}

\begin{figure}[t]
\centering
\includegraphics[width=\textwidth]{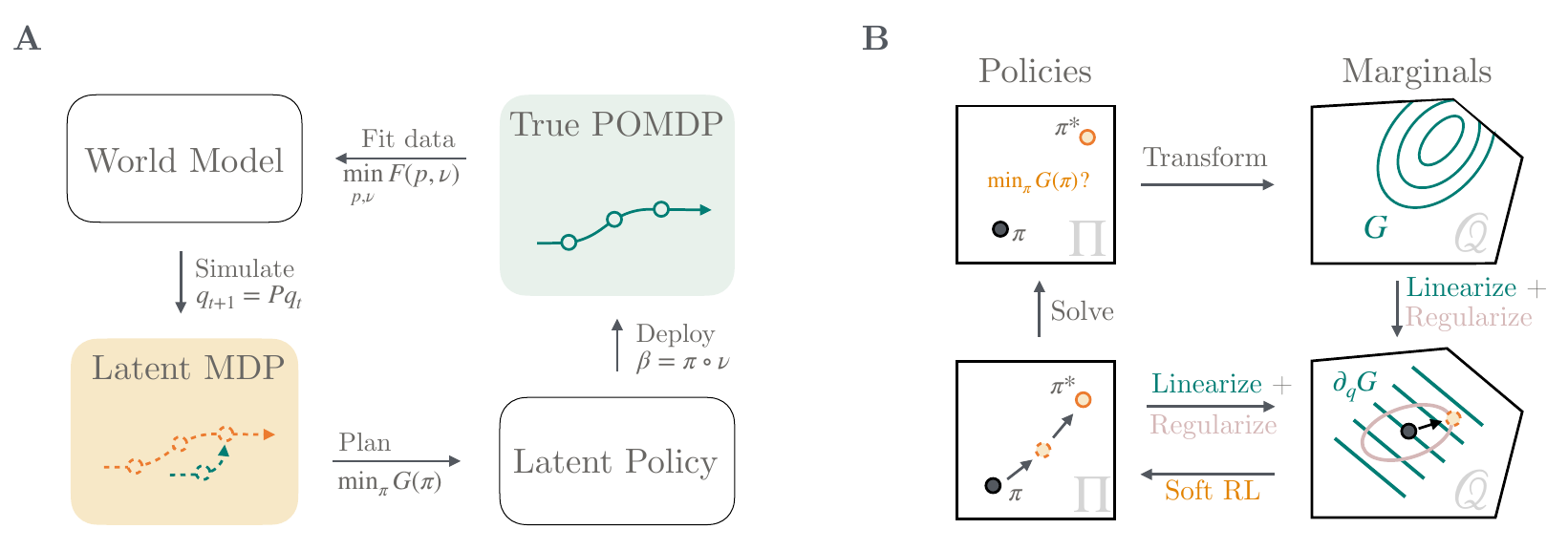}
\caption{The episodic active inference setting considered in this work: A) the environment admits a true partially observable decision process, while the agent learns a latent fully observable process using variational inference. B) We consider a convex optimization approach where policies are transformed into their marginals, turning EFE minimization into a convex Markov decision process (MDP) which can be solved by a series of soft reinforcement learning (Soft RL) updates.}
\label{fig:setting}
\end{figure}

Model-based policy optimization (MBPO) lets an agent learn a world model from
environment interaction and then optimize a reactive policy on imagined rollouts
inside that model. In reinforcement
learning (RL), the properties of MBPO are well understood: algorithms converge
under suitable assumptions, and policy-improvement guarantees can be stated for
the expected reward~\cite{janner2019mbpo,luo2019slbo,delgrange2026deepspi}.

Active Inference (AIF)~\cite{dacosta2020,friston2015active} can also be understood as
MBPO. The agent learns a generative model of the environment, then selects actions by minimizing an expected free energy (EFE) that
scores imagined future outcomes against a preference distribution.
Implementations use similar architectural components as recent world model-based RL methods~\cite{ha2018worldmodels,hafner2020,mazzaglia2021}: a recurrent world model, an actor, and a critic, trained on imagined environment interactions. The optimization properties of EFE minimization, however, are far less understood than those of
reward maximization. 
This blocks the transfer of efficient algorithms, and their convergence and policy-improvement guarantees, from the RL literature to the simulation of AIF agents.

One reason is that AIF is formulated with inference, not policy optimization, in mind. To bring it
within reach of the RL toolbox, we make two modifications to the standard
protocol (cf.~\cite{dacosta2020}). First, we take an \emph{episodic} view: rather
than updating on every environment step, VFE and EFE minimization each run over
one or more full episodes while the other's target is held fixed, see \cref{fig:setting}. This is only a mild departure since standard AIF already freezes the model and variational posterior
during planning when it imagines future timesteps. 
Second, in place of the common mean-field posterior over the future we use a controlled hidden Markov model.

With these two modifications in place, our central claim is that EFE
minimization under a fixed model is a \emph{convex MDP}%
~\cite{zahavy2021reward,zhang2020variational}, the recent generalisation of RL in
which the objective can be an arbitrary convex functional of the state--action marginal at future timesteps rather than a linear one (as with expected reward). The per-step EFE is exactly of this form. Its extrinsic-value
and ambiguity terms are linear in the occupancy, while the epistemic value, the
negative state-marginal entropy, supplies a single convex term such that the whole of convex-MDP theory applies. At the exact-posterior limit this epistemic term coincides with the mutual information between states and observations, recovering the familiar information-gain view of active inference. We return to this in \cref{sec:convex}. In particular, convex MDPs admit extended
dynamic-programming methods that remain scalable and sample-efficient: we solve EFE minimization by a mirror-descent algorithm~\cite{beck2003mirror,moreno2024efficient} whose subproblem combines soft RL~\cite{haarnoja2018soft} and maximum-entropy exploration~\cite{hazan2019maxent}, is solved in closed form by softmax dynamic programming~\cite{ziebart2010purposeful,neu2017unified,geist2019theory}, and converges at rate $O(1/K)$ by relative smoothness of the EFE~\cite{bauschke2017descent,lu2018relatively}.

Beyond the fixed-model case, we show that allowing the agent to refit the model from data of the real environment significantly complicates the optimization problem. Active inference sits at the underexplored intersection of convex MDPs and performative reinforcement learning~\cite{perdomo2020,mandal2023,rank2024}, where the reward and the environment dynamics respond to the agent's deployed policy.

\section{Related work}

\emph{Clarifying AIF.}
Like control-as-inference in RL, active inference models action selection as variational
inference~\cite{rawlik2012stochastic,levine2018inference,millidge2020relationship,imohiosen2020unifying},
adding an epistemic drive that evidence objectives generally
lack~\cite{millidge2021origin,millidge2021whence} and that weights entropy
differently from marginal
inference~\cite{lazarogredilla2024planning}. AIF and RL have also been compared
empirically~\cite{sajid2021demystified} and theoretically through the lens of
Bellman-optimality~\cite{friston2021sophisticated,dacosta2023reward}. 
Recent work relates the EFE to Bayes-optimality via value of
information~\cite{wei2024value}, $\rho$-POMDPs~\cite{cooper2026belief},
augmented-model inference~\cite{nuijten2026type,nuijten2026planning}, and
Bethe free energy~\cite{kouw2026bethe}. While these clarify the \emph{inference}
structure of the EFE, we clarify its \emph{optimization} structure.
Finally,~\cite{champion2026reframing} show that the different EFE formulations in the literature are not equivalent in general. We use the
observation-preference class (risk over observations / information gain) formulation as an example, but our results apply to all of them.

\emph{Implementations.}
Deep AIF agents use tree
search~\cite{fountas2020montecarlo}, bootstrapped EFE
critics~\cite{millidge2020variational}, or joint policy--model
updates~\cite{taheriyeganeh2025delayed}, but EFE objectives remain hard to
scale~\cite{mazzaglia2021} and can collapse to single-action
policies~\cite{champion2024deconstructing}. We predict a class of algorithms that compute exact \emph{natural} gradients of the EFE by soft Bellman recursion.

\emph{Convex MDPs and performative RL.}
Convex MDPs optimize nonlinear occupancy
functionals~\cite{zahavy2021reward,geist2022concave,hazan2019maxent,lee2019smm};
additive information bonuses keep the objective
linear~\cite{houthooft2016vime}. Finite-trial and expected objectives
differ~\cite{mutti2022challenging}; our prior mirror-descent
work~\cite{milosevic2026geometry} treats a fixed environment. Performative
RL~\cite{mandal2023,mandal2025linear,hardt2025performative} separates
stable from optimal policies under distribution
shift~\cite{chen2025optimal,basu2026pepg}, but---like objective-mismatch
analyses of MBPO~\cite{rajeswaran2020game,lambert2020objective}---assumes
linear reward. Active inference needs both nonlinear objectives and
performative dynamics.

\section{
Episodic Active Inference
}\label{sec:setting}
We consider a finite-horizon reward-free POMDP
$\mathcal{M} = (\cS, \cA, \cO, P, E, \sigma)$, where $s_t\in\cS, a_t\in\cA$, and
$o_t\in\cO$ are states, actions, and observations at time
$t\in[T]=\{0,\dots,T\}$, with $\cS,\cA,\cO$ finite. $P$ and $E$ are transition and emission kernels with full support and $\sigma$ is a start state distribution. 

We work in the memoryless finite-horizon, tabular setting in the main text (dropping parameter indexes following~\cite{agarwal2021theory}); the history-dependent and infinite-horizon cases extend our results and are deferred to \S\ref{sec:extensions}. A (time-dependent) behavior policy $\beta=(\beta_t)_{t=0}^{T-1}$, $\beta_t:\cO\to\Delta(\cA)$,
induces the environment trajectory law
\begin{equation}\label{eq:env-traj}
p^\beta(\tau)=\sigma(s_0)E(o_0 | s_0)\prod_{t=0}^{T-1}\,\beta_t(a_t | o_t)P(s_{t+1} | s_t,a_t)E(o_{t+1} | s_{t+1}).
\end{equation}

An \emph{active inference agent} consists of a variational world model $m=(p,\nu)$ deployed by chaining latent-state inference and action selection.\footnote{This approach is inspired by the Dreamer \cite{hafner2020} architecture.} Throughout, its behavior policy is the composition $\beta=\pi\circ\nu$,
$\beta_t(a_t | o_t)=\sum_{s_t}\pi_t(a_t | s_t)\,\nu(s_t | o_t)$:
at each step the agent infers the latent state through the recognition density $\nu$ and acts through the latent-state-conditioned policy $\pi_t$ optimized by the EFE phase. The policy is time-dependent because finite-horizon planning returns a per-step closed-loop plan $\pi=(\pi_t)_{t=0}^{T-1}$. For simplicity, the planning horizon $T$ is the episode length, so $\pi_t$ is the rule the agent applies at episode step $t$. Each $\pi_t(\cdot | s)\in\Delta(\cA)$ is a full-support distribution over actions for every state $s\in\cS$, so the \emph{latent policy space} is
\begin{equation}\label{eq:policy-space}
  \Pi=\prod_{t=0}^{T-1}\big(\Delta(\cA)\big)^{\cS},
\end{equation}
the set of all such time-indexed, state-conditioned policies. Crucially, $\Pi$ does not depend on the model.

We consider an episodic model-based policy-optimization setting in which the
agent alternates between
\begin{enumerate}[leftmargin=*]
    \item \emph{Model learning:} explore the environment using $\beta$, collect
    trajectory data $\cD$ over several episodes, and fit the world model
    $m=(p,\nu)$ by minimizing variational free energy (VFE) on the
    data,
    \item \emph{Policy optimization:} improve the behavior policy by
    minimizing expected free energy (EFE) on imagined trajectory data with
    respect to the latent policy $\pi$, using the latest world model.
\end{enumerate}

\noindent One full pass through the active inference loop is
\begin{equation}\label{eq:full-loop}
     \pi\xmapsto{\ \mathrm{deploy}\ }\cD(\pi)\xmapsto{\ \text{VFE}\ }m(\cD(\pi))\xmapsto{\ \text{EFE}\ }\pi\big(m(\cD(\pi))\big) =: \mathrm{AIF}(\pi),
\end{equation}
and a fixed point satisfies $\pi^*=\mathrm{AIF}(\pi^*)$. The VFE and EFE steps can either be run to completion or interleaved episodically.
Since the latent policy space $\Pi$~\eqref{eq:policy-space} does not depend on the
model, the natural question is under which conditions
repeated application of $\mathrm{AIF}$ admits such a fixed point and whether efficient algorithms exist that converge to it.

\paragraph{Execution.}
In the real environment, the model $m=(p,\nu)$ stays fixed throughout an episode (it is refit only between episodes, in the VFE phase). At episode step $t$ the agent \emph{applies} the frozen recognition map $\nu$ to the \emph{actual} observation $o_t$ and acts through the state-conditioned $\pi_t$, i.e.\ $a_t\sim\beta_t(\cdot | o_t)$. It thus follows a state-conditioned \emph{policy} indexed by elapsed episode time, not a fixed action sequence and not the imagined belief propagated under its own model. This closed-loop, state-conditioned form is exactly what makes the imagined law~\eqref{eq:imag-traj} Markov, so its marginals trace the marginal polytope of \S\ref{sec:convex}.

\subsection{VFE: World Model Training}\label{sec:vfe-main}
The model $m=(p,\nu)$ that the EFE phase plans against is the output of a preceding
perception/learning phase: the agent runs $\beta=\pi\circ\nu$ in the real environment, collects
history data $\cD(\pi)$, and fits $m$ by minimizing variational free energy, $m^*(\pi)=\arg\min_m\cF(m;\cD(\pi))$ (assumed unique, \cref{app:vfe}).
Three features of this phase are relevant for EFE minimization: (i) $\cF$ decomposes into \emph{surprise} plus a non-negative \emph{recognition gap} $\E_\cD\,\KL(\nu\,\|\,p(\cdot | o,a))$
that vanishes exactly when $\nu$ equals the model posterior. This same gap governs posterior accuracy, so a
VFE-optimal $\nu$ is what makes the EFE phase score the epistemic drive accurately. (ii) $m$ carries
the same recognition density $\nu$ used by the behavior policy, so imagined and online inference
are consistent (\S\ref{sec:setting}). (iii) the fitted model $m^*(\pi)$ generally \emph{depends on the
deploying policy} $\pi$ through the data it induces. This coupling will make the closed loop
performative (\S\ref{sec:performative}). The full generative model, recognition factorization, and VFE
decomposition are given in \cref{app:vfe}.

\subsection{EFE: Planning in a Latent MDP}
Given a model $m=(p,\nu)$, a latent policy $\pi$ induces the purely imagined trajectory law
\begin{equation}\label{eq:imag-traj}
\tilde q^\pi(\tau;m)=\tilde q_0(s_0)\,p(o_0 | s_0)\prod_{t=0}^{T-1}\pi_t(a_t | s_t)\,p(s_{t+1} | s_t,a_t)\,p(o_{t+1} | s_{t+1}),
\end{equation}
where the initial belief $\tilde q_0$ is given and assumed full-support ($\tilde q_0(s)>0$). For example, it can be an additional factor $\tilde q_0(s_0) = \nu(s_0)$ in the variational posterior or derived from the recognition density $\tilde q_0(s_0) = \nu(s_0 | o_0)$ where $o_0\sim\cD$ is sampled from a dataset (cf. Dreamer~\cite{hafner2020}).

Given an active inference agent with model $m$, a preference model $\tilde p$, and a latent trajectory distribution $\tilde  q^\pi(\cdot;m)$, define the finite-horizon EFE by
\begin{align}\label{eq:efe-general}
\cG(\pi;m)&:=\sum_{t=0}^T G_t(\pi;m) = \sum_{t=0}^T\E_{\tilde q^\pi}\!\big[\,\log q_t^\pi(s_t)-\log\tilde p(s_t,o_t)\,\big],
\end{align}
where $q_t^\pi$ is the predictive state marginal of the imagined law \cref{eq:imag-traj}. 
EFE variants primarily diverge in how they factorize the preference model $\tilde p$~\cite{dacosta2020,parr2020markov,millidge2021whence}. A common factorization is $\tilde p(s_t,o_t)\propto \nu(s_t | o_t)\,p(o_t | s_t)\,\check p(o_t)$, where $\nu$ is the agent's recognition density, especially useful in deep policy optimization of the EFE~\cite{mazzaglia2021,mazzaglia2022free}.
We define the planning (policy optimization) problem as
\begin{equation}\label{eq:efe-min}
    \pi^*(m) \in \arg \min_\pi \mathcal{G}(\pi;m).
\end{equation}

\paragraph{The epistemic value as mutual information.}
Under \cref{eq:efe-general} and the above factorization, $\mathbb{E}[\log q_t^\pi(s)-\log\nu(s_t|o_t)]$ appears in place of the mutual information because the recognition density $\nu$ is held fixed. If $\nu$ is replaced by the exact model posterior $p(s | o)$ (attained at VFE convergence whenever the recognition family is rich enough to contain it, cf.\ \cref{app:vfe}), this term becomes the negative mutual information $-I_t(S;O)$ between states and observations at time $t$, the classical epistemic/information-gain value of active inference~\cite{parr2020markov,millidge2021whence}. Our convex MDP results below are preserved for any of the common interpretations, as long as $\tilde p$ can be assumed constant-in-policy during planning.

\section{
EFE Minimization is a Convex MDP
}\label{sec:convex}
Our main structural result is that EFE minimization under a fixed model $m=(p,\nu)$ is a convex MDP, shown by writing it not as a policy functional, but over its induced \emph{state-action marginals} and exhibiting the convex program.

\subsection{Planning as a Convex MDP}
To expose the convex structure, we transform the policies into their induced state-action marginals.\footnote{The RL analogue is the \textit{dual linear program} formulation of MDPs~\cite{puterman1994}, where the transformed object is the policy's \emph{state-action occupancy}.} 
Recall the imagined trajectory law (\ref{eq:imag-traj}), which models planning as a controlled Markov process. Due to the Markov property, the step-t marginals $q_t^\pi(s,a)=\mathbb{P}_{\tilde q^\pi}(s_t=s,a_t=a)$ of the planning process fully determine the step-$t$ terms of the EFE $G_t$ in Equation \ref{eq:efe-general}. They further satisfy the recursion
\begin{align}
\sum_{a'}q_{t+1}^\pi(s',a') &= \sum_{s,a} p(s' | s,a)\,q_t^\pi(s,a),\quad t=0,\dots,T-2,\label{eq:flow}\\ 
\sum_a q_0^\pi(s,a) &= \tilde q_0(s),\qquad q_T^\pi(s'):=\sum_{s,a}p(s' | s,a)\,q_{T-1}^\pi(s,a)\label{eq:boundary}
\end{align}
with the terminal state marginal $q_T$ obtained by one further push-forward and $\tilde q_0$ is the start state from \cref{eq:imag-traj}. Intuitively, \cref{eq:flow} is conservation of probability flow and \cref{eq:boundary} are boundary conditions. The joint state-action belief $q_t$ (not the state-only belief) turns the problem into a convex program: the map $\pi\mapsto q^\pi:=(q_t^\pi)_{t=0}^{T-1}$ is a \emph{bijection} onto a convex set, where a policy is recovered as $\pi_t(a | s)=q_t(s,a)/q_t(s)$ (conditioning), cf. \cite{puterman1994}, and the EFE turns into a convex functional as we show below. 

Crucially, the flow constraints \cref{eq:flow} confine the feasible marginals to the polytope~\cite{puterman1994,moreno2024efficient}
\begin{equation}
    \mathcal{Q}=\{(q_t)_{t=0}^{T-1}\in(\Delta_{\cS\times\cA})^{T}:\eqref{eq:flow}-\eqref{eq:boundary}\text{ hold}\},
\end{equation} a convex subset of the simplex over $\mathcal{S}\times\mathcal{A}$. Every $q\in\mathcal Q$ recovers a unique policy $\pi_t(a | s)=q_t(s,a)/\sum_{a'}q_t(s,a')$ wherever $q_t(s)>0$. On the interior of $\mathcal{Q}$ (where every $\pi_t(a | s)>0$) the full support of $P$ and of $q_0$ makes $q_t(s)>0$ for all $s,t$, so this inverse is defined everywhere.
The EFE can be defined on this polytope such that $\mathcal{G}(\pi) = \Gamma(q^\pi) + $const. and takes the compact form
\begin{equation}\label{eq:gamma-noac}
    \Gamma(q^\pi)\coloneqq\;\langle\ell,q^\pi\rangle- \mathcal{\bar H}(q^\pi),\quad \ell_t(s,a):=-\mathbb{E}_{\substack{s'\sim p(\cdot | s,a)\\ o\sim p(\cdot | s')}}\log\epref(s',o)
\end{equation}
where $\langle\cdot,\cdot\rangle$ is the Euclidean inner product and $\mathcal{\bar H}(q)=\sum_{t=1}^{T}\mathcal{H}(\rho_t^\pi)$ is the sum of state-marginal entropies over the arrival marginals $\rho_1,\dots,\rho_T$. As sometimes done in RL, the step-$t$ cost $\ell_t(s,a)$ is the preference of the \emph{arrival} state $s_{t+1}$, so the terminal preference at $t=T$ is decided at state-action marginal $q_{T-1}$. The residual $\mathrm{const}=\sum_s\tilde q_0(s)\,\ell(s)-\mathcal H(\tilde q_0)$ collects the $t{=}0$ terms of the initial marginal $\rho_0=\tilde q_0$, which no policy can influence and which therefore does not affect the optimization. Because $\pi\mapsto q^\pi$ is a bijection, the policy problem $\min_\pi\cG(\pi)$ and the problem $\min_{\mu\in\mathcal{Q}}\Gamma(q)$ on its induced marginals are equivalent.

\begin{proposition}[EFE minimization is a convex MDP]\label{prop:convex}
The finite-horizon EFE minimization problem of Equation \ref{eq:efe-min} under a fixed model $m$ is a finite-horizon convex MDP, meaning 
\begin{equation}\label{eq:convex-program}
    \min_{\pi\in\Pi}\cG(\pi) \equiv \min_{q\in\mathcal{Q}}\Gamma(q),
\end{equation}
and \cref{eq:convex-program} is a convex program over policy-induced state-action marginals, i.e.\ a convex MDP~\cite{zahavy2021reward,moreno2024efficient}.
\end{proposition}
Optimization problems of this kind are referred to as convex~\cite{zahavy2021reward}, general-utility~\cite{zhang2020variational,santos2025the}, or concave-utility MDPs~\cite{moreno2024efficient} in the RL literature. Specifically, EFE minimization has the structure of maximum-entropy exploration~\cite{hazan2019maxent} combined with standard reward maximization, here in the episodic/finite-horizon setup studied by~\cite{moreno2024efficient}.
While standard MDPs also have a convex program representation, the objective there is linear, and the nonlinear case does not admit global value functions. However, specialized dynamic programming methods exist, which we will exploit in the algorithmic development of \S\ref{sec:md-aif}.

\subsection{History-Dependent and Infinite-Horizon Extensions}\label{sec:extensions}
The finite-horizon memoryless setting is for clarity only. The convex-MDP structure is preserved in both natural extensions. For history-dependent policies one augments the state with the history, $x=(h,s)$, or a compressed information state (analogous to the imagined state of Dreamer~\cite{hafner2020,mazzaglia2021}), and the same convex program results (\cref{app:hist}). For infinite horizons the objective must aggregate time, and since the entropy operator does not commute with time aggregation, one must choose how state and time averages combine. The discounted and average-reward variants, their Jensen gap, and their solvers are developed in \cref{app:infinite} (see also~\cite{mutti2023convex,santos2025the} for further details on estimation bias in finite trials that result from the Jensen gap).

\section{
EFE Minimization via Soft Reinforcement Learning
}\label{sec:md-aif}

\subsection{The EFE admits local rewards}

While the EFE, being a convex functional of time-$t$ marginals, does not admit a value function, its gradient does and obeys a backward recursion. The backward gradient messages are then interpretable as a \emph{reward prediction error}. Define the
\emph{gradient-to-go} of the EFE $U_t$ as the negative functional derivative of the EFE with respect to the step-t marginal, $U_t(s,a):=-\,\delta\Gamma/\delta q_t(s,a)$. Its local (per-step) part exhibits the form of a reward function
\begin{align}\label{eq:local-error}
  r_t^\pi(s,a)
  \;&=\;\E_{s'\sim p(\cdot | s,a)}\Big[\underbrace{\E_{o\sim p(\cdot | s')}\log\epref(s',o)}_{\text{preference}}
  \;-\;\underbrace{\big(\log q_{t+1}^\pi(s')+1\big)}_{\text{predictive-marginal error}}\Big].
\end{align}
It is the preference reward minus the novelty/marginal error $\log q_{t+1}^\pi(s')$ that the convex
epistemic term contributes for the current policy, both evaluated at the arrival state $s'$ and transported by $p(\cdot | s,a)$ as a reward $r(s,a)=\E_{s'}[r(s')]$. 

Threading this local error through all downstream marginals by the
chain rule gives the backward recursion, for $t=T-1,\dots,0$,
\begin{align}\label{eq:value-rec}
  U^{\pi}_t(s,a)\;&=\;r^{\pi}_t(s,a)+\E_{s'\sim p(\cdot | s,a)}\bar U^{\pi}_{t+1}(s'),\\
  \bar U^{\pi}_t(s)\;&=\;\E_{a'\sim\pi_{t}(s)}U^{\pi}_{t}(s,a'),\quad \bar U^{\pi}_{T}(s)=0,
\end{align}
where the base case vanishes because the arrival-state reward $r^\pi_{T-1}$ already
accounts for the terminal marginal $q_T=P^{\!\top} q_{T-1}$. Its single-transition sample is a reward prediction error (RPE)
\begin{equation}\label{eq:rpe}
  \delta^{\pi}_t\;=\;r^{\pi}_t(s_t,a_t)+U^{\pi}_{t+1}(s_{t+1},a_{t+1})-U^{\pi}_t(s_t,a_t),
\end{equation}
for $t=0,\dots,T-1$, an \textit{advantage estimate} in RL terms \cite{schulman2018highdimensionalcontinuouscontrolusing}. This is the temporal-future analogue of the cortical prediction error, i.e. the free-energy gradient
transported \emph{backward in time} rather than \emph{up} the hierarchy. Everything specific to the convex geometry is deferred to \cref{app:md}.

\begin{algorithm}[t]
\caption{MD-AIF}
\label{alg:mdaif}
\begin{algorithmic}[1]
\Require world model $m=(p,\nu)$; preference $\tilde p_t$; horizon $T$; step size $\eta$; init. belief $q_0$; current policy $\bar\pi = (\bar\pi_t)_{t=0}^{T-1}$.
\Ensure $\text{MD-AIF}(\pi^{k}; m, \tilde p_t, T, \eta, q_0)$ computes the update \cref{eq:md}.
  \LComment{Forward pass: imagined state-marginal under $\bar\pi$ for reward}
  \State $\rho_0 \gets q_0$;
  \For{$t=0,\dots,T-1$}
     \State $\rho^{\bar\pi}_{t+1}(s')\gets\sum_{s,a}p(s' | s,a)\,\bar\pi(a | s)\,\rho^{\bar\pi}_t(s)$
  \EndFor
  \For{$t=0,\dots,T-1$, all $s,a$}
     \State $r^{\bar\pi}_t(s,a)\gets\mathbb{E}_{s'\sim p(\cdot | s,a)}\big[\underbrace{\mathbb{E}_{o}\log \tilde p(s',o)}_{\text{preference}}\;\underbrace{-\,\log \rho^{\bar\pi}_{t+1}(s')-1}_{\text{novelty}}\big]$
  \EndFor
  \LComment{Backward pass: soft (free-energy) value iteration}
  \State $V_{T}(s)\gets 0$
  \For{$t=T-1,\dots,0$}
     \State $Q_t(s,a)\gets r_t(s,a)+\sum_{s'}p(s' | s,a)\,V_{t+1}(s')$
     \State $\pi_t(a | s)\gets\dfrac{\bar\pi_t(a | s)\exp\!\big(\eta\, Q_t(s,a)\big)}{\sum_{a'}\bar\pi_t(a' | s)\exp\!\big(\eta\, Q_t(s,a')\big)}$
     \State $V_t(s)\gets\sum_a \pi_t(a|s)
    \Bigl[Q_t(s,a)
    - \eta^{-1}\log\frac{\pi_t(a|s)}{\bar\pi_t(a|s)}\Bigr]$
  \EndFor
\State \Return $\pi_t$ for all $t$
\end{algorithmic}
\end{algorithm}

\subsection{Solution via soft backward recursion}
It is known that entropy-regularized (soft) reinforcement learning can be derived from backward message passing on a probabilistic graphical model~\cite{levine2018inference}. Here, we show that this idea applies to the EFE as well, but only locally, meaning w.r.t to its current prediction errors.
\cref{alg:mdaif} performs per iteration a Boltzmann policy update like in soft RL
\begin{equation}\label{eq:efe-blr}
    \pi_t^{k+1}(a | s)\;\propto\;\pi_t^k(a | s)\,\exp\!\big(\eta\,Q_t^{{k+1}}(s,a)\big),
\end{equation}
where $\eta$ is a temperature/step-size and $Q_t^{k+1}$ is the soft action-value function that backs up the (linearized) prediction error for the current policy~\eqref{eq:local-error}. Hence $Q_t$ is the entropy-regularized form of the gradient-to-go $U_t$ and converges to it as the
step contracts, with the softmax increment the policy-dependent RPE~\eqref{eq:rpe}. 
The appeal of this approach is that one obtains well-conditioned natural gradients on the free energy which can be computed efficiently with controlled bias-variance trade-off and over long horizons.

The natural gradient correspondence follows from writing the policy as an exponential family, $\pi_t(a | s)\propto\exp\langle w_t,\phi(s,a)\rangle$, where $\phi$ is a vector of one-hot-encoded $(s,a)$.
The natural gradient update $w\leftarrow w-\eta\,F^{-1}\nabla_w\cG$ with $F$ the
policy Fisher information~\cite{kakade2001natural} is the \emph{Bayesian learning rule} (BLR;~\cite{khan2023bayesian,khan2017conjugate}), which is the same rule that gives perceptual inference when the exponential family is instead the perceptual belief~\cite{kataoka2026extendedpredictivecodingframework}. On the policy family it is exactly the softmax update of \cref{alg:mdaif}.

Equivalently, the update is the \emph{mirror descent} algorithm (MD;~\cite{beck2003mirror}) applied to the EFE on marginals, making \cref{alg:mdaif} the active-inference instance of MD-CURL~\cite[Alg.~2]{moreno2024efficient}, which we use to state the convergence result below.

\begin{restatable}[Convergence of MD-AIF]{proposition}{PropRatePsi}\label{prop:rate-psi}
Run \cref{alg:mdaif} from any full-support policy under a fixed model $m$, with horizon $T$ and a
constant step size $\eta=1/L$, $L=\tfrac12 T(T+1)$. Then the EFE of the iterates converges to its
minimum at rate
\[
\min_{0\le k\le K}\cG(\pi^k;m)-\cG(\pi^\star;m)\ =\ O(1/K).
\]
\end{restatable}

Repeated soft policy iteration drives the EFE to its global optimum. The step size is set by the horizon through
$L=\tfrac12T(T+1)$, which grows quadratically in $T$. However, the bound is loose and the step size can be larger in practice. The rate follows because the EFE is well-conditioned relative to the policy relative entropy\footnote{We can not use the MD-CURL rate directly, because the state entropy is not Lipschitz continuous, however, it is relatively smooth w.r.t the regularizer.} used in the softmax
update (relative smoothness, cf. \cite{lfn2018}). The constant, the convex geometry, and the full proof are in \cref{app:md}.

\subsection{Performative aspects of active inference}\label{sec:performative}
The fixed-model results above assume the planner sees a model that does not change as the
policy changes. Closing the active inference loop~\eqref{eq:full-loop} breaks that
assumption: after deploying $\pi$ closed-loop (executing $\beta=\pi\circ\nu$ against the environment, \S\ref{sec:setting}) the agent refits its world model to the data $m_\pi := m(\cD(\pi))$ that this
deployment induced, so the local pseudo-reward $r_{q^\pi}$ and the dynamics $P_{q^\pi}$ that the next planning phase optimizes against are \emph{functions of the deployed policy}. This is exactly
the structure of performative reinforcement learning~\cite{perdomo2020,mandal2023,rank2024}.

\begin{assumption}[Smoothness of performative map; cf. \cite{mandal2023}]\label{a:sensitivity}
    We assume the model-learning loop is $(\varepsilon_r,\varepsilon_p)$-sensitive: for all
    marginals $q,q'$,
    \[
      \|r_q-r_{q'}\|\le\varepsilon_r\,\|q-q'\|,
      \qquad
      \|P_q-P_{q'}\|\le\varepsilon_p\,\|q-q'\|.
    \]
    This is~\cite[Assumption~1]{mandal2023}, applied to the refit map
    $q\mapsto m^\star(\cD(\pi^q))$.
\end{assumption}

Note that $r_q$ depends on the policy anyway through the re-linearization of the EFE, but now also through the previously linear part in $\Gamma$, since it depends on the learned model dynamics $m_\pi$. We leave sufficient conditions on the model class for Assumption \ref{a:sensitivity} for future work.

The \emph{performative update operator} is the single step of \cref{alg:mdaif}:
\begin{equation}\label{eq:md-operator}
  \mathcal{P}(\pi)_t(a | s)\;:=\;
  \frac{\pi_t(a | s)\,\exp\!\big(\eta\,Q^\pi_t(s,a)\big)}
       {\sum_{b}\pi_t(b | s)\,\exp\!\big(\eta\,Q^\pi_t(s,b)\big)},
\end{equation}
where $\pi\mapsto Q^\pi$ has additional dependencies through $\ell_\pi$ and $P_\pi$, which is why we need the continuity assumption above.

\begin{proposition}[Existence]\label{prop:exist-brouwer}
Under Assumption \ref{a:sensitivity}, $\mathcal{P}$ has a fixed point.
\end{proposition}

\begin{proof}
Apply Brouwer's theorem on the closure $\bar\Pi$ (meaning we include deterministic policies), which is a nonempty compact convex polytope. The map $\mathcal P$ of~\eqref{eq:md-operator} is
single-valued and maps $\bar\Pi$ into itself, since each $\mathcal{P}(\pi)_t(\cdot | s)$ is an explicit normalized nonnegative vector in $\Delta(\cA)$. Further, it is \emph{continuous}: the softmax is smooth, $\pi\mapsto q^\pi$ is continuous by the forward recursion, and $q\mapsto(r_q,P_q)$ is continuous by Assumption \ref{a:sensitivity}. The backward recursion producing $Q^\pi$ runs through the varying kernel $P_q$, but it is a composition of finitely many ($T$) maps that are jointly continuous, so $\pi\mapsto Q^\pi$ is continuous. It is finite because the full support of $p$ and $\tilde q_0$ keeps $\rho_{t}>0$ and hence $\log\rho_{t}$ bounded on $\bar\Pi$. Brouwer's theorem gives
$\exists\pi^\star\in\bar\Pi: \pi^\star=\mathcal{P}(\pi^\star)$.
\end{proof}

Existence does not imply i) that the retraining iterates \emph{reach} such a point rapidly or at all, and ii) that a given fixed point has the desired properties, e.g.\ the model converges to the true environment dynamics. We instead hypothesize a simplified regime in which a guarantee could hold and discuss the active-inference feature that obstructs it. 

\paragraph{Toward a convergence guarantee.}
A clean case for full AIF convergence is i) well-specified tabular parametrization ii) with exact-recognition, and iii) state coverage under all deployed policies. If all three hold, the deployed policy only changes how frequently each transition is seen, never which are seen or what the model converges to.

However, the coverage assumption is unlikely to hold under diverse behavior policies in complex environments, where AIF is arguably most promising. The epistemic drive in the EFE actively pushes mass onto under-visited states and so helps coverage naturally, but it does not certify a uniform minimum of state coverage.
We therefore leave a convergence analysis to future work, where natural routes are reduction to online learning via dataset aggregation~\cite{ross11} or the mixed delayed repeated retraining of~\cite{rank2024}. The experiments of \cref{sec:experiments} instead probe the closed loop with the simplifying assumptions.

\section{Experiments}\label{sec:experiments}

We test the structural predictions of the convex-MDP view on gridworlds: that \cref{alg:mdaif} converges at the $O(1/K)$ rate of
\cref{prop:rate-psi} with the policy-dependent reward driving
broad state coverage through the epistemic term, and that optimizing the
epistemic term by soft dynamic programming identifies the environment faster
than gradient descent and myopic planning. Environment, hyperparameter, and
baseline details are in App. \ref{app:experiments}.\footnote{Code reproducing
both panels: \url{https://github.com/milosen/mdaif-experiments}.}

\textbf{Planning convergence.}
On a $5\times5$ grid with a Manhattan-distance preference, we compare MD-AIF with
entropy-regularized RL and Euclidean gradient descent on the EFE.
MD-AIF attains a lower EFE and tracks the theoretical $O(1/K)$ line
(\cref{fig:gridworld}A, left), and its planned marginals stay spread across the
grid, in contrast to the direct channel to the goal taken by RL
(\cref{fig:gridworld}A, right).

\begin{figure}[t]
  \centering
  \includegraphics[width=\textwidth]{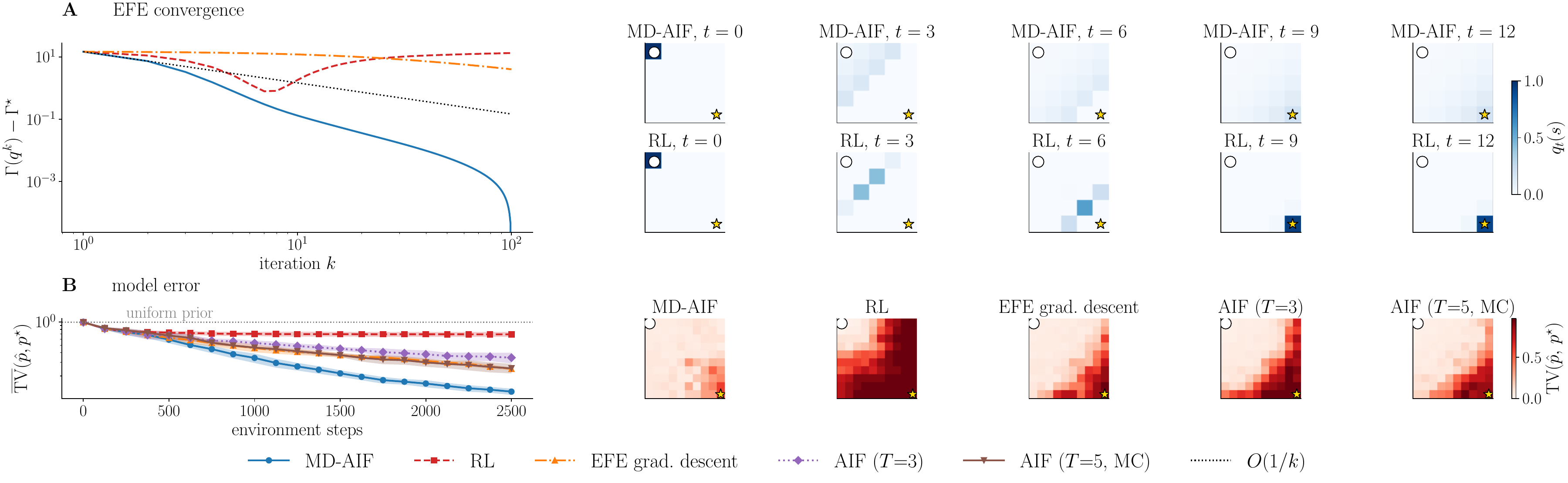}
  \caption{%
    \textbf{MD-AIF on gridworld environments.}
    \textbf{(A)} $5\times5$ deterministic grid. EFE convergence (left) at the $O(1/k)$ rate, and imagined occupancy $q_t(s)$ at convergence (right); MD-AIF keeps mass spread across the grid, RL channels it along the direct path to the goal.
    \textbf{(B)} $10\times10$ stochastic POMDP: slippery dynamics, with a single central state whose emission is ambiguous (confusable with its grid neighbours). The observation preference is uniform, leaving the epistemic term as the only drive. Closed-loop model learning from rollouts: model error $\overline{\mathrm{TV}}(\hat{p},p^\star)$ vs.\ environment steps (left; mean over $5$ seeds, bands are $95\%$ CIs) and per-state error after 2500 env. steps (right, seed-averaged). MD-AIF's information gain term broadens coverage fastest and learns quickest, ahead of EFE gradient descent and myopic AIF~\cite{dacosta2020} ($T\!=\!3$ exact, $T\!=\!5$ Monte Carlo). RL, lacking the epistemic term, stays at the uniform policy.
  }
  \label{fig:gridworld}
\end{figure}

\textbf{Model learning.}
We next make the task partially observed: a slippery $10\times10$ grid in which a single central state has an ambiguous emission, confusable with its grid neighbours while every other state emits exactly. The agent never sees the state but
filters a recognition belief, acts through it, and refits $\hat p$ from belief-weighted soft counts. Interleaving policy optimization with model refitting, the performative loop of MD-AIF yields broader coverage and faster reduction of the model error $\overline{\mathrm{TV}}(\hat p,p^\star)$ than
EFE gradient descent or myopic EFE planning, while
reward-only RL, lacking the epistemic term, stays close to the prior
(\cref{fig:gridworld}B). Over $5$ seeds the gap to every baseline is
significant at the final budget ($p\le0.003$), and the advantage holds over most of the trajectory rather than only at its end.

\section{Conclusion}\label{sec:conclusion}
For closed-loop policies under a fixed model, expected free energy minimization is a convex MDP:
the pragmatic, ambiguity, and recognition terms form a linear latent-MDP reward, and the
epistemic value is the only nonlinearity, a convex negative state-marginal entropy. This places
EFE in the convex MDP class and makes its algorithms available, notably MD, which coincides with
the natural policy gradient on the EFE and converges at its standard $O(1/K)$ relative-smoothness rate. 
Refitting the model and the recognition density to policy-induced data closes the loop and makes it performative. Here we prove that the closed loop admits a fixed point (\cref{prop:exist-brouwer}) and leave to future work both convergence to it and whether it is stable in the performative sense. 

\newpage
\appendix
\section*{Appendix}

\section{Notation and conventions}\label{app:notation}
In the main text, we overloaded the notation for $q$ to mean the imagined trajectory distribution, the infered time-0-marginal and the step-t state marginals and state-action marginals, which is common in the active inference literature~\cite{dacosta2020}. Yet it may lead to confusion here, since the distinctions are important for the convex optimization theory. Hence, we write $\mu_t(s,a):=q_t(s,a)$ for the state-action marginals and $\rho_t(s):=q_t(s)$ for the state-only marginals. 

Define the occupancy entropies by
\begin{align}
\Phi(\mu)&=\sum_{t}\sum_{s'}(P\mu_t)(s')\log (P\mu_t)(s')&\text{(arrival-state negative entropy)},\\
\Phi_0(\mu)&=\sum_{t,s,a}\mu_t(s,a)\log \sum_{a'}\mu_t(s,a')&\text{(state negative entropy)},\\
\Psi(\mu)&=\sum_{t,s,a}\mu_t(s,a)\log\frac{\mu_t(s,a)}{\sum_{a'}\mu_t(s,a')}&\text{(conditional negative entropy)},\\
R(\mu)&=\sum_{t,s,a}\mu_t(s,a)\log \mu_t(s,a)&\text{(state-action negative entropy)}.
\end{align}
Note that $R=\Phi_0+\Psi$, and that $\Phi(\mu)=-\sum_{t=1}^{T}\mathcal H(\rho_t)$ ranges over the arrival marginals $\rho_1,\dots,\rho_T$.

Throughout we write
\[
\Gamma(\cdot)\coloneqq\langle\ell,\cdot\rangle+\Phi(\cdot),
\]
so that the general EFE is
$$\cG(\pi;m)=\Gamma(\mu^{\pi})=\langle\ell,\mu^{\pi}\rangle+\Phi(\mu^{\pi}).$$
A full reference for symbols used in the appendix can be found in Table~\ref{tab:notation-app}.

\begin{table}[ht]
\centering\small
\setlength{\tabcolsep}{6pt}
\caption{Notation used throughout, including the convex-analysis symbols kept out of the main text.}
\label{tab:notation-app}
\begin{tabularx}{\linewidth}{@{}lX@{}}
\toprule
Symbol & Meaning\\
\midrule
$\rho_t^\pi(s)$ & step-$t$ state marginal of $q^\pi$\\
$\mu_t^\pi(s,a)$ & state--action occupancy, $\mu_t^\pi(s,a)=\rho_t^\pi(s)\,\pi_t(a | s)$\\
$\ell_t(s,a)=-\E_{s'\sim p(\cdot | s,a),\,o}\log\epref(s',o)$ & linear latent-MDP cost\\
$\Gamma(\mu)=\langle\ell,\mu\rangle+\Phi(\mu)$ & EFE in marginal form\\
$\Phi(\mu)=-\sum_{t=1}^{T}\mathcal H(\rho_t)$ & arrival-state marginal negative entropy\\
$\Psi(\mu)=-\sum_{t}\rho_t(s)\mathcal H(\pi_t)(s)$ & conditional negative entropy (MD generator)\\
$R(\mu)=\Phi_0(\mu)+\Psi(\mu)$ & state--action negative entropy\\
$r_t^\pi(s,a)=-\partial\Gamma/\partial\mu_t(s,a)$ & local per-step reward / prediction error, \cref{eq:local-error}\\
$U_t^\pi(s,a)=-\delta\Gamma/\delta\mu_t(s,a)$ & gradient-to-go: total $\Gamma$-derivative \cref{eq:value-rec}\\
$\bar U_t^\pi(s)=\sum_a\pi_t(a | s)\,U_t^\pi(s,a)$ & state contraction of the gradient-to-go\\
$\delta_t$ & single-transition reward prediction error, \cref{eq:rpe}\\
\bottomrule
\end{tabularx}
\end{table}

\section{Details on VFE minimization}\label{app:vfe}
This appendix gives the world-model (VFE) phase in full.

\paragraph{Generative model and recognition.}
With actions treated as given, the agent's generative model over a trajectory
$\tau=(s_0,o_0,a_0,\dots,a_{T-1},s_T,o_T)$ factorizes as
\begin{equation}\label{eq:gen-model}
p(o_{0:T},s_{0:T} | a_{0:T-1})=p(s_0)p(o_0 | s_0)\prod_{t=0}^{T-1}p(o_{t+1} | s_{t+1})p(s_{t+1} | s_t,a_t),
\end{equation}
and states are inferred with the recognition density
\[
  \nu(s_{0:T} | o_{0:T},a_{0:T-1})=\prod_{t=0}^{T}\nu(s_t | o_t).
\]
Under tabular parametrization we drop the parameter indexes and consider the memoryless case; the history-dependent case is discussed in \cref{sec:extensions}.

\paragraph{VFE.}
Let $\cD(\pi)$ be the distribution of histories $h_T=(o_{0:T},a_{0:T-1})$ obtained
by running $\beta=\pi\circ\nu$ in the real environment. The variational free
energy of $m$ on data $\cD$ is
\begin{equation}\label{eq:vfe-surprise}
\cF(m;\cD)=\underbrace{\E_{h\sim\cD}\big[-\log p(o_{0:T} | a_{0:T-1})\big]}_{\text{surprise}}
+\underbrace{\E_{h\sim\cD}\,\KL\big(\nu(s_{0:T} | h)\,\|\,p(s_{0:T} | h)\big)}_{\text{recognition gap}\,\ge\,0},
\end{equation}
where $p(s_{0:T} | o,a)$ is the exact Bayesian posterior under~\eqref{eq:gen-model}. The bound is tight precisely when $\nu$ equals that posterior. Under the mean-field factorization above this is generally unattainable, since the smoothing posterior couples the $s_t$ across the history: VFE minimization returns the closest $\nu$ \emph{within that family} and the recognition gap stays bounded away from zero. Widening the family to the history-conditioned $\nu(s_t | h_t)$ of \cref{app:hist} makes the exact posterior attainable and the bound tight. In either case the
non-negative recognition gap is the same quantity that controls EFE accuracy, so VFE-optimality of $\nu$ is what makes the EFE phase score against an accurate posterior. The model-learning phase targets
\begin{equation}\label{eq:vfe-min}
    m^*(\pi) = \arg\min_m\ \cF(m;\cD(\pi)),
\end{equation}
where the optimum $m^*$ generally depends on the policy $\pi$ used to generate the trajectory data via $\beta=\pi\circ\nu$. We assume throughout that this minimizer is unique, so that $\pi\mapsto m^*(\pi)$ is a well-defined map; this is what lets the refit be treated as a function of the deployed policy in \cref{sec:performative}. Uniqueness can fail, e.g.\ when the deployed policy leaves some transitions unvisited so that the data do not constrain them, in which case $m^*$ must be read as a fixed selection from the minimizing set.

\section{EFE minimization is a convex MDP}\label{app:efe}
Here, we develop the occupancy form of the general EFE and show that EFE minimization has the form of a convex MDP, which means the optimization problem is convex minimization in the space of state-action occupancies (state-action marginals).

Using~\cref{eq:imag-traj} as the trajectory distribution, the
agent plans by acting in imagination and confines optimization to the
dynamically feasible occupancies (those satisfying the flow
constraints below). 
This is also the structure which allows us to write EFE minimization as a convex MDP.

\begin{lemma}\label{lem:gamma-convex}
We write $\mu^\pi = (\mu_t^\pi)_{t=0}^{T-1}$ for the $T$ per-step occupancies, with state marginals $\rho_t^\pi=\sum_a\mu_t^\pi(\cdot,a)$ for $t<T$ and the derived terminal marginal $\rho_T^\pi=\sum_{s,a}p(\cdot | s,a)\mu_{T-1}^\pi(s,a)$. The EFE \cref{eq:efe-general} can be written as
\begin{equation}\label{eq:efe-occ}
\cG(\pi;m)
=
\sum_{t=0}^{T-1}\left\{\sum_{s,a}
\mu_t^\pi(s,a)
\Big[
{-}\,\E_{\substack{s'\sim p(\cdot | s,a)\\ o\sim p(\cdot | s')}}\log \tilde p(s', o)
\Big]
-\mathcal H\big(P\mu_t^\pi\big)\right\}
\;+\;\mathrm{const},
\end{equation}
or compactly on marginals
\begin{equation}\label{eq:gamma-occ-app}
    \cG(\pi)= \Gamma(\mu^\pi) := \;\langle\ell,\mu^\pi\rangle+\Phi(\mu^\pi),
\end{equation} where $\ell \coloneqq \ell_t(s,a)=-\E_{s'\sim p(\cdot|s,a),\,o\sim p(\cdot|s')}\log\epref(s',o)$ is the expected preference of the \emph{arrival} state, $\Phi(\mu^\pi) = -\sum_{t=0}^{T-1}\mathcal H(P\mu^\pi_t)$ is the neg. marginal Shannon entropy (including the derived terminal marginal), $\mathrm{const}=\sum_s\tilde q_0(s)\,\ell(s)-\mathcal H(\tilde q_0)$ collects the policy-independent $t{=}0$ terms of the initial marginal, and
$\langle\cdot,\cdot\rangle$ is the Euclidean inner product on
$\mathbb R^{\cS\times\cA\times\{0,\dots,T-1\}}$.
\end{lemma}

\begin{proof}
Starting from~\eqref{eq:efe-general}, the one-step expected free energy at time $t$ depends only on the marginal of the trajectory at $t$:

\begin{equation}
G_t(\pi;m)
=
\sum_{s \in \mathcal S}
\sum_{o \in \mathcal O}
\rho_t^\pi(s)p(o|s)
\left[
\log \rho_t^\pi(s)
-
\log \tilde p(s, o)
\right].
\end{equation}
Sum over timesteps on both sides. Re-index by the transition that produces each marginal: for $t\ge1$, $\rho_t^\pi(s')=\sum_{s,a}p(s'|s,a)\mu_{t-1}^\pi(s,a)$, so
\[
\sum_{t=1}^{T}\sum_{s'}\rho_t^\pi(s')\,\ell(s')
=\sum_{t=0}^{T-1}\sum_{s,a}\mu_t^\pi(s,a)\,\E_{s'\sim p(\cdot|s,a)}\ell(s')
=\sum_{t=0}^{T-1}\langle\ell_t,\mu_t^\pi\rangle,
\]
which is the arrival-state cost. The remaining $t=0$ terms, $\sum_s\tilde q_0(s)\ell(s)$ and $-\mathcal H(\tilde q_0)$, are independent of $\pi$ and are collected into $\mathrm{const}$. This accounts for all $T{+}1$ pragmatic terms of~\eqref{eq:efe-general} while keeping the pragmatic sum over the $T$ genuine occupancies.
\end{proof}

\begin{lemma}
    $\Gamma$ is convex.
\end{lemma}

\begin{proof}
    $-\cH$ is convex as a function of $\rho_t$, which is a linear image of $\mu_{t-1}$, and a composition of a linear map with a convex function is convex, $-\cH(\rho_t^\pi)$ is convex in $\mu_{t-1}^\pi$, but not strictly convex. The inner products are linear (hence, convex) in $\mu_t$ whenever $-\log \tilde p$ is independent of $\mu_t$. Each summand is convex in the aggregated variable $\mu=(\mu_t)_{t=0}^{T-1}$ (it is convex in one block and constant in the others), so their sum $\Gamma$ is jointly convex in $\mu$.
\end{proof}

Now the proof of the statement from the main text. It remains to be shown that $\min_\pi \mathcal{G}(\pi)$ is a convex MDP.

\begin{proof}[of \cref{prop:convex}]
$\Gamma$ is convex. Further, $\mathcal Q$ is a convex polytope since the flow constraints \cref{eq:flow} are linear in $\mu$. Thus \cref{eq:convex-program} is a convex program in state-action marginals (occupancies), by definition a convex MDP~\cite{zahavy2021reward}. More precisely, it is the finite-horizon convex MDP (concave-utility) problem of~\cite{moreno2024efficient}.
\end{proof}

\section{Iterated Soft RL solves EFE planning}\label{app:md}

To show that Algorithm \ref{alg:mdaif} converges to the EFE minimum at the predscribed rate, we proceed in 3 steps: 1) show that Algorithm \ref{alg:mdaif} is equivalent to the mirror descent (MD) algorithm on $\mathcal Q$, 2) show that MD's relative smoothness condition holds for the EFE objective 3) apply the resulting relative-smoothness rate~\cite{lfn2018} to the EFE.

\subsection{Preliminary: The gradient-to-go obeys a backward recursion}\label{app:chainrule}

As a preliminary, we justify the chain-rule claim of \S\ref{sec:md-aif}: threading the local
prediction error~\eqref{eq:local-error} through all downstream marginals yields
the backward recursion~\eqref{eq:value-rec}. Two derivatives must be
distinguished. The \emph{local} reward $r^\pi_t$ is the partial derivative of
the single term $G_t$ with the occupancies treated as free variables, whereas
the \emph{gradient-to-go} $U^\pi_t$ is the derivative of the whole objective
$\cG$ along the flow, i.e.\ accounting for the fact that perturbing $\mu_t$
displaces every later marginal $\mu_{t'}$, $t'>t$, through~\eqref{eq:flow}. The
recursion is exactly the statement that the second is obtained from the first by
transporting it backwards with the model kernel.

\begin{lemma}[Local prediction error]\label{lem:local-error}
The partial derivative of $\Gamma$ at a feasible $\mu^\pi$ induced by $\pi$, taken with the
coordinates $\mu_t(s,a)$ treated as free, is
\begin{equation}\label{eq:partial-gamma}
  \frac{\partial\Gamma(\mu^\pi)}{\partial\mu_t(s,a)}
  \;=\;\E_{s'\sim p(\cdot | s,a)}\Big[\E_{o\sim p(\cdot | s')}\log\epref(s',o)-\big(\log\rho^\pi_{t+1}(s')+1\big)\Big].
\end{equation}
\end{lemma}

\begin{proof}
The term $\langle\ell,\mu\rangle=\sum_{t,s,a}\ell_t(s,a)\mu_t(s,a)$ is linear and
contributes $\ell_t(s,a)$. For $\Phi$, by the
chain rule for the linear marginalisation map,
$\partial_{\mu_t}\mathcal{H}(\rho_{t+1})
=\big(\log\rho_{t+1}+1\big)\cdot\partial_{\mu_t}\rho_{t+1}
=\log\rho_{t+1}+1=\log P\mu_t+1$.
\end{proof}

Both parts of the prediction error
are evaluated at $s_{t+1}$ and transported by $p(\cdot | s,a)$, as for a reward
$r(s,a)=\E_{s'}[r(s')]$ in RL. This makes $r^\pi_t$ state--action dependent and
lets $r^\pi_{T-1}$ absorb both $t{=}T$ terms of $\Gamma$, so the recursion starts
from $\bar U^\pi_T\equiv0$.

\begin{proposition}[Backward recursion for the EFE
gradient]\label{prop:chain-rule}
Fix a full-support policy $\pi$ with marginals $\mu^\pi\in\mathcal Q$, and let
$r^\pi_t$ ($t=0,\dots,T-1$) be as in \cref{lem:local-error}. Define the terminal gradient $\bar U^\pi_T(s):=0$ and,
backwards for $t=T-1,\dots,0$,
\begin{align}\label{eq:value-rec-app}
\bar U^\pi_{t+1}(s')&:=\sum_{a'}\pi_{t+1}(a' | s')\,U^{\pi}_{t+1}(s',a'),\\
  U^{\pi}_t(s,a)\;&:=\;r^{\pi}_t(s,a)
  +\sum_{s'}p(s' | s,a)\,\bar U^\pi_{t+1}(s').
\end{align} Then $U^\pi_t$ is the total functional derivative of the EFE with respect to the occupancy along the flow constraints,
\begin{equation}\label{eq:grad-identity}
  \frac{\delta\Gamma[\mu^\pi]}{\delta\mu_t}(s,a)\;=\;-\,U^\pi_t(s,a),\qquad t=0,\dots,T-1,
\end{equation}
as opposed to the \emph{partial} derivative $\partial\Gamma/\partial\mu_t=-r^\pi_t$ (\cref{lem:local-error}), which sees only the explicit step-$t$ dependence.
\end{proposition}

\begin{proof}
Note that the flow is linear, so $\delta \rho_t=P\delta\mu_{t+1}$ and by the product rule $\delta\mu_t = \pi_t\delta\rho_t + \rho_t\delta\pi_t$. Recall that we have defined $U_t = r_t + P^T\bar U_{t+1}$. Now, compute the first variation (Gateaux derivative)
\begin{align}
    \delta\Gamma[\mu^{\pi'}](\delta \mu_{t'}) &= -\sum_{t=t'}^{T-1} \sum_{s,a}\frac{\partial\Gamma(\mu^{\pi'})}{\partial\mu_t(s,a)} \delta\mu_{t}(s,a) \\
    &= -\sum_{t=t'}^{T-1} \langle r_t^{\pi'},\delta\mu_t\rangle = -\sum_{t=t'}^{T-1} \langle r_t^{\pi'},\delta\mu_t\rangle -\underbrace{\langle\bar U_T,\delta \rho_T\rangle}_{=0}.
\end{align}
Aggregating the partial derivatives backwards in time yields the result.
\end{proof}

\subsection{Softmax Value Backup is Mirror Descent}

Algorithm 1 solves $\min_{\pi}\cG(\pi)$, by mirror descent (Bregman proximal gradient descent)
with the conditional-negative-entropy generator $\Psi$, meaning that per iteration it solves
\begin{equation}\label{eq:md}
    \mu^{k+1}\in\arg\min_{\mu\in\mathcal{Q}}\big\{\langle\nabla\Gamma(\mu^k),\mu\rangle+\eta_k^{-1}D_\Psi(\mu\,\|\,\mu^k)\big\},
\end{equation}
where
\begin{equation}
D_\Psi(\mu\|\mu^k)=\sum_{t,s}\rho^\mu_t(s)\KL(\pi^\mu_t(\cdot | s)\|\pi^{\mu^k}_t(\cdot | s))
\end{equation}
is the Bregman divergence generated by $\Psi$.
We use $\Psi$ throughout: it yields the
closed-form softmax/DP update below. The proof follows \cite[Theorem 4.1]{moreno2024efficient}.

\begin{proposition}[Soft value backup is MD]\label{thm:dp}
Let $\mu^k \in \mathcal{Q}$ be the current occupancy iterate with
conditional policy $\pi^k_t(a|s) := \mu^k_t(s,a)/\rho^k_t(s)$.
The mirror step~\eqref{eq:md} with generator $\Psi$ and step size
$\eta_k$ is solved in closed form by the following interleaved
backward pass:

\noindent\textbf{Initialise:} $V_{T}^k(s) \equiv 0$.

\noindent\textbf{For $t = T\!-\!1, \ldots, 0$:}
\begin{align}
    \label{eq:Q-recursion}
  Q_t^{k+1}(s,a) &= r_t^k(s,a) + \sum_{s'}p(s'|s,a)\,V_{t+1}^k(s'),\ r_t^k = -\nabla\Gamma(\mu^k)_t\\
  \pi_t^{k+1}(a|s) &= \frac{\pi_t^k(a|s)\exp\!\bigl(\eta_k\,Q_t^{k+1}(s,a)\bigr)}
    {\sum_{a'}\pi_t^k(a'|s)\exp\!\bigl(\eta_k\,Q_t^{k+1}(s,a')\bigr)}\label{eq:policy-update}\\
    V_t^{k+1}(s) &= \sum_a \pi_t^{k+1}(a|s)
    \Bigl[Q_t^{k+1}(s,a)
    - \eta_k^{-1}\log\frac{\pi_t^{k+1}(a|s)}{\pi_t^k(a|s)}\Bigr].\label{eq:V-expectation}
\end{align}
\end{proposition}

\begin{proof}
Write $\mu_t(s,a) = \rho_t(s)\pi_t(a|s)$. The subproblem~\eqref{eq:md}
becomes:
\begin{equation}\label{eq:sub-expanded}
\min_{\mu \in \mathcal{Q}}\;
  \sum_{t,s,a}\mu_t(s,a)\Bigl[-r_t^k(s,a)
  + \eta_k^{-1}\log\frac{\pi_t(a|s)}{\pi_t^k(a|s)}\Bigr],
\end{equation}
where $\mathcal{Q}$ is the set of occupancy measures satisfying
the flow constraints.
Introduce multipliers $V_t(s)$ for the flow
constraints at $t = 0,\ldots,T$ and assemble the Lagrangian (the normalization of each $\pi_t(\cdot | s)$ needs no separate multiplier: it is implied by the $t$-th flow constraint, and $V_t(s)$ below is exactly the induced log-partition):
\begin{align}
L(\mu,V) &= \sum_{t,s,a}\mu_t(s,a)
  \Bigl[-r_t^k(s,a)
  + \eta_k^{-1}\log\frac{\pi_t(a|s)}{\pi_t^k(a|s)}\Bigr]
  \notag\\
&\quad + \sum_{t=1}^{T}\sum_sV_t(s)
  \Bigl[\sum_a\mu_t(s,a)
  - \sum_{s',a'}\mu_{t-1}(s',a')\,p(s|s',a')\Bigr]
  \notag\\
&\quad + \sum_sV_0(s)
  \Bigl[\sum_a\mu_0(s,a) - \mu_0(s)\Bigr].
  \label{eq:lagrangian}
\end{align}
Setting $\frac{\partial L}{\partial\mu_t(s,a)} = 0$:
\begin{equation}\label{eq:stationarity}
-r_t^k(s,a) + \eta_k^{-1}\log\frac{\pi_t(a|s)}{\pi_t^k(a|s)}
+ V_t(s) - \sum_{s'}\,p(s'|s,a)V_{t+1} = 0\quad \forall (t,s,a).
\end{equation}
Define $Q_t(s,a) := r_t^k(s,a) + \sum_{s'}p(s'|s,a)V_{t+1}(s')$.
Then~\eqref{eq:stationarity} becomes:
\begin{equation}\label{eq:stat-compact}
\eta_k^{-1}\log\frac{\pi_t(a|s)}{\pi_t^k(a|s)}
= Q_t(s,a) - V_t(s).
\end{equation}
Exponentiating~\eqref{eq:stat-compact}:
\begin{equation}\label{eq:pi-unsolved}
\pi_t(a|s) = \pi_t^k(a|s)
  \exp\!\bigl(\eta_k[Q_t(s,a) - V_t(s)]\bigr).
\end{equation}
Summing over $a$ and solving for $V_t(s)$:
\begin{equation}\label{eq:lambda-solved}
V_t(s) = \eta_k^{-1}\log\sum_a\pi_t^k(a|s)
  \exp\bigl(\eta_k Q_t(s,a)\bigr).
\end{equation}
Substituting back into~\eqref{eq:pi-unsolved} gives the policy
update:
\begin{equation}\label{eq:policy-final}
\pi_t(a|s) = \frac{\pi_t^k(a|s)
  \exp\bigl(\eta_k Q_t(s,a)\bigr)}
  {\sum_{a'}\pi_t^k(a'|s)
  \exp\bigl(\eta_k Q_t(s,a')\bigr)}.
\end{equation}

Define $V^{k+1} = V^*, Q^k=r^k+P^T V^{k+1}$ (policy evaluation). Since $V^*$ at each $t$ depends on $V^*$ at $t+1$ only
(through $Q_t$), the system is solved by a single backward
sweep initialised at $t=T$ and running $t = T\!-\!1, \ldots, 0$:
\begin{align}
V_T^{k+1}(s) &= 0,\\
Q_t^{k+1}(s,a) &= r_t^k(s,a)
  + \sum_{s'}p(s'|s,a)\,V_{t+1}^{k+1}(s'),
  \quad t=T\!-\!1,\dots,0,\label{eq:Q-final}\\
\pi^{k+1}_t(a|s) &= \frac{\pi_t^k(a|s)
  \exp\bigl(\eta_k Q_t^{k+1}(s,a)\bigr)}
  {\sum_{a'}\pi_t^k(a'|s)
  \exp\bigl(\eta_k Q_t^{k+1}(s,a')\bigr)}\\
V_t^{k+1}(s) &= \eta_k^{-1}\log\sum_a\pi_t^{k+1}(a|s)
  \exp\bigl(\eta_k\,Q_t^{k+1}(s,a)\bigr).
  \label{eq:V-final}
\end{align}
At each $t$, the optimal multiplier $V_t^{k+1}$ is determined by
$Q_t^{k+1}$ (which depends on $V_{t+1}^k$) and the known reference
$\pi^k$.
Using~\eqref{eq:stat-compact} at the optimum, $\log\frac{\pi_t^{k+1}}
{\pi_t^k} = \eta_k(Q_t^{k+1} - V_t^{k+1})$, so:
\begin{align}
 V_t^{k+1}(s) = \sum_a\pi_t^{k+1}(a|s)\Bigl[Q_t^{k+1}(s,a)
  - \eta_k^{-1}\log\frac{\pi_t^{k+1}(a|s)}{\pi_t^k(a|s)}\Bigr].
\label{eq:V-expectation-derived}
\end{align}
\end{proof}

\begin{remark}[Soft value $V$ vs.\ exact gradient-to-go $U$]\label{rem:V-vs-U}
The optimal multiplier
$V_t^{k+1}(s)=\sum_a\pi_t^{k+1}(a | s)\bigl[Q_t^{k+1}(s,a)
-\eta_k^{-1}\log\frac{\pi_t^{k+1}(a|s)}{\pi_t^k(a|s)}\bigr]$ evaluates the KL-regularized subproblem at its optimum $\pi^{k+1}$.
This is the \emph{soft} counterpart of the on-policy recursion
$\bar U^\pi_t(s)=\sum_a\pi_t(a | s)\,U^\pi_t(s,a)$ of the gradient-to-go
(\cref{prop:chain-rule}).
\end{remark}

\subsection{Relative smoothness of the objective}

\begin{lemma}[Markov contraction]\label{lem:chi2}
For a Markov kernel $K:\cY\to\cZ$, $\nu>0$ on $\cY$, signed $v$ on $\cY$: $\ \|Kv\|_{K\nu}\le\|v\|_\nu$.
\end{lemma}
\begin{proof}
$(Kv)(z)^2=\big(\sum_y K(z | y)v(y)\big)^2\le (K\nu)(z)\sum_y K(z | y)\tfrac{v(y)^2}{\nu(y)}$ by
Cauchy--Schwarz; divide by $(K\nu)(z)$, sum over $z$, use $\sum_z K(z | y)=1$.
\end{proof}

\begin{proposition}[Relative smoothness]\label{thm:relsmooth}
The EFE $\Gamma$ is $L$-smooth relative to $\Psi$ on $\mathcal{Q}$ with $L=\tfrac12 T(T+1)$.
\end{proposition}

\begin{proof} 
Relative smoothness for twice differentiable $\Gamma$ and $\Psi$ is the condition $\ \nabla^2\Gamma(\mu)\preceq \tfrac12 T(T+1)\,\nabla^2\Psi(\mu)$ on the tangent space of $\mathcal{Q}$, see~\cite{lfn2018}. 

A Bregman divergence is linear in its convex potential, meaning $D_{af_1+bf_2} = a D_{f_1} + b D_{f_2}$. Applied to the joint negentropy $R$ this yields
\[
\underbrace{\textstyle\sum_t\KL(\mu_t\Vert\mu_t')}_{D_R}
=\underbrace{\textstyle\sum_t\KL(\rho_t\Vert\rho_t')}_{D_{\Phi_0}}
+\underbrace{\textstyle\sum_{t,s}\rho_t(s)\,\KL\!\big(\pi_t(\cdot\mid s)\Vert\pi_t'(\cdot\mid s)\big)}_{D_\Psi}.
\]
The KL-divergence obeys
$\KL(x,x+\delta x)=\tfrac12\|\delta x\|_x^2+o(\|\delta x\|^2)$, where $\|v\|^2_x=\sum_{i}\frac{v(i)^2}{x(i)}$ is the fisher length at $v$. Expanding each term to second order yields the Pythagorean split of the per-step Fisher metric and identifies all three quadratic forms at once:
\begin{equation}\label{eq:fisher-split}
\underbrace{\textstyle\sum_{t=0}^{T-1}\|\delta\mu_t\|^2_{\mu_t}}_{\nabla^2 R[\delta\mu]}
=\underbrace{\textstyle\sum_{t=0}^{T-1}\|\delta\rho_t\|^2_{\rho_t}}_{\nabla^2\Phi_0[\delta\mu]}
+\underbrace{\textstyle\sum_{t=0}^{T-1}\sum_{s}\rho_t(s)\,\|\delta\pi_t(\cdot\mid s)\|^2_{\pi_t}}_{\nabla^2\Psi[\delta\mu]},
\end{equation}
which holds per step $t=0,\dots,T-1$ (where a genuine occupancy $\mu_t$ splits into marginal and conditional). The only
nonlinear part of $\Gamma$ is $\Phi=-\sum_{t=1}^{T}\mathcal H(\rho_t)$, and $\Phi$ depends on $\mu$ only through
the linear marginal maps $\rho_t(\cdot)=\sum_a\mu_t(\cdot,a)$ ($t<T$) and the derived terminal $\rho_T$, so its Hessian has no conditional part and includes the terminal marginal,
\[
\nabla^2\Gamma[\delta\mu]=\sum_{t=1}^{T}\|\delta\rho_t\|^2_{\rho_t}.
\]
Thus $\nabla^2\Gamma$ and $\nabla^2\Psi$ are the marginal and conditional blocks of the same per-step
Fisher metric. Relative smoothness is the statement that $L\nabla^2\Psi$ upper bounds $\nabla^2\Gamma$ for some finite $L$.
 
The marginal variation $\delta\rho_t$ is fully determined by the conditional components $\{\rho_{t'}\delta \pi_{t'}\}_{t'<t}$ through the flow map~\cref{eq:flow}. Writing
$w_{t'}:=\rho_{t'}\!\otimes\!\delta\pi_{t'}$ for the conditional component of $\delta\mu_{t'}$ (so
$\sum_{s,a}w_{t'}=\sum_s\rho_{t'}\sum_a\delta\pi_{t'}=0$), and linearising $\rho_{t+1}=P_{\pi_t}\rho_t$
with $\delta\rho_0=0$ and $P_\pi = \pi\circ P$ gives:
\[
\delta\rho_t=\sum_{t'<t}K_{t'\to t}\,w_{t'},\qquad
K_{t'\to t}:=P_{\pi_{t-1}}\cdots P_{\pi_{t'+1}}\,
P\vphantom{P_{\pi}} .
\]
As a composition of row-stochastic kernels, $K_{t'\to t}$ is itself a Markov kernel. Applying
\cref{lem:chi2},
\[
\|K_{t'\to t}\,w_{t'}\|_{\rho_t}\;\le\;\|w_{t'}\|_{\mu_{t'}}
=\sqrt{\textstyle\sum_s\rho_{t'}(s)\,\|\delta\pi_{t'}(\cdot | s)\|^2_{\pi_{t'}}}.
\]
The triangle and Cauchy--Schwarz inequalities over the $t$ sources give
\[
\|\delta\rho_t\|^2_{\rho_t}\le\Big(\textstyle\sum_{t'<t}\|K_{t'\to t}w_{t'}\|_{\rho_t}\Big)^2
\le t\sum_{t'<t}\sum_s\rho_{t'}(s)\,\|\delta\pi_{t'}(\cdot | s)\|^2_{\pi_{t'}}
\]
and summing over $t$ with $\sum_{t=0}^{T}t=\tfrac12T(T+1)$ gives
\[
\begin{aligned}
\nabla^2\Gamma[\delta\mu]=\sum_t\|\delta\rho_t\|^2_{\rho_t}
&\le\tfrac12T(T+1)\sum_{t'}\sum_s\rho_{t'}(s)\,\|\delta\pi_{t'}(\cdot | s)\|^2_{\pi_{t'}}\\
&=\tfrac12T(T+1)\,\nabla^2\Psi[\delta\mu],
\end{aligned}
\]
i.e.\ $\Gamma$ is $L$-smooth relative to $\Psi$ on the flow-polytope tangent with $L=\tfrac12T(T+1)$.
\end{proof}

\begin{proof}[of \cref{prop:rate-psi}]
    The softmax update solves the mirror descent update on $\Gamma$ with generator $\Psi$ exactly and keeps each iterate interior, so the relative smoothness bound \cite[Thm.~3.1]{lfn2018} with $\eta=1/L$ holds along the trajectory as long as $\langle\nabla^2\Gamma\rangle\leq L\langle\nabla^2\Psi\rangle$, which holds by Proposition~\ref{thm:relsmooth}.
\end{proof}

\section{History dependent and infinite horizon formulations}
The episodic memoryless development of \cref{app:efe,app:md} transfers to history-dependent policies and posterior densities, as well as to the
discounted and average-reward settings with appropriate extensions of the standard EFE definition.

\subsection{History-dependence}\label{app:hist}
Consider the finite-horizon reward-free POMDP $\mathcal{M} = (\cS, \cA, \cO, P, E, \sigma)$ as before. Let the history at time $t$ be $h_t=(o_{0:t},a_{0:t-1})$, updated by $h_{t+1}=h_t\cdot(a_t,o_{t+1})$ with $h_0=(o_0)$, and let $\cH=\bigcup_{t<T}(\cA\times\cO)^t\times\cO$ be the space of histories. The history-dependent behavior policy $\beta:\cH\times\cT\to\Delta(\cA)$ is of the form $\beta(a_t|h_t) =\sum_{s_t}\pi(a_t|h_t,s_t)\nu(s_t|h_t)$ and induces the environment trajectory law
\begin{equation}\label{eq:env-traj-hist}
p^\beta(\tau)=\sigma(s_0)E(o_0 | s_0)\prod_{t=0}^{T-1}\,\beta(a_t | h_t)P(s_{t+1} | s_t,a_t)E(o_{t+1} | s_{t+1}).
\end{equation}
Given a model $m=(p,\nu)$ with history-dependent variational density $\nu(s_t|h_t)$, a latent policy $\pi(a_t|h_t,s_t)$ induces the purely imagined trajectory law
\begin{equation}\label{eq:imag-traj-hist}
q^\pi(\tau;m)=q_0(s_0)\,p(o_0 | s_0)\prod_{t=0}^{T-1}\pi_t(a_t | h_t,s_t)\,p(s_{t+1} | s_t,a_t)\,p(o_{t+1} | s_{t+1}),
\end{equation}
with predictive history-state marginals $\rho^\pi_t(s,h)=\mathbb{P}_\pi(s_t=s,h_t=h)$.

For history-dependent variational density $\nu(s_t | h_t)$, define the augmented state space
$\cX=\cH\times\cS$ and the augmented kernel
\begin{equation}
    p(x' | x,a)=p\big((s',h') |(s,h),a\big)=\mathbf 1[h'=h\cdot(a,o')]\,p(o' | s')\,p(s' | s,a).
\end{equation}
Then the imagined process~\eqref{eq:imag-traj-hist} satisfies
\begin{equation}\label{eq:recursion-hist-app}
    \rho_0^\pi=q_0,\qquad \rho_{t+1}^\pi(x')=\sum_{x,a}p(x' | x,a)\,\pi_t(a | x)\,\rho_t^\pi(x),
\end{equation}
where $\rho^\pi_t((s,h))=\mathbb{P}(s_t=s,h_t=h)$ are the history-state marginals of the trajectory distribution $q$ at time $t$.
Then this gives a history-dependent EFE
\begin{equation}\label{eq:efe-occ-hist}
\mathcal E(\pi;m)
=
\sum_{t=0}^T\sum_{x \in \mathcal X}\sum_{a \in \mathcal A}
\pi_t(a | x)\,\rho_t^\pi(x)
\left[
\log \rho_t^\pi(x)
-
\sum_{o \in \mathcal O}p(o|s)\log \tilde p(x)
\right],
\end{equation}
where we overload notation and write $\log \tilde p(x) = \log \tilde p(s, o)$ as a density on $\cX$. Since \cref{eq:recursion-hist-app} makes the imagined process Markov on $\cX$, the proof of \cref{prop:convex} applies verbatim with $s\to x$, so $\min_\pi\mathcal E(\pi;m)$ is again a convex MDP, now on the augmented state space $\cX$.

The history component of $x\in\cX$ makes $|\cX|$ grow with $t$. In a function-approximation implementation, $h$ is typically summarised by a
recurrent state, e.g.\ write $x=f(h,s)$. Replace $s \to x$ and the forward pass in \cref{alg:mdaif} with an entropy-gradient estimate (e.g. kNN-based~\cite{ashlag2026}) and the tabular sweeps of \cref{alg:mdaif} generalize to a Dreamer-like architecture. The memoryless setting
drops $h$ entirely and recovers a standard $\cS$-state
convex-MDP solver.

\subsection{Infinite-horizon formulations}\label{app:infinite}
Extending the EFE to the infinite horizon introduces a choice absent in the finite-horizon case: how the future is aggregated inside the nonlinear entropy term (cf.\ \cref{tab:infinite_horizon_efe}). Throughout we work in the stationary regime where $\ell$ and $\pi$ do not depend on time.
Applying the entropy per step gives the term-by-term analogue of the finite-horizon EFE, for example for a discounted aggregation, 
$$\cG_\gamma^{\mathrm{step}}(\pi)=\sum_{t\ge0}\gamma^t[\langle\ell,\mu_t^\pi\rangle-\mathcal H(\mu_t^\pi(x,a))].$$ 

Writing $d^\pi=(1-\gamma)\sum_t\gamma^t\mu_t^\pi$ for the mixture, concavity of $\mathcal H$ and Jensen's inequality give $\mathcal H(d^\pi)\ge(1-\gamma)\sum_t\gamma^t\mathcal H(\mu_t^\pi)$. Hence $\mathcal G_{\gamma}^{\mathrm{marg}}\neq \mathcal G^{\mathrm{step}}$ in general, see \cref{tab:infinite_horizon_efe}. The convex MDP interpretation only holds for $\mathcal G_{\gamma}^{\mathrm{marg}}$ in the infinite horizon, where the soft RL analogy still holds.

\begin{table}[t]
\centering
\small
\setlength{\tabcolsep}{6pt}
\renewcommand{\arraystretch}{1.3}
\caption{Infinite-horizon extensions of the EFE objective. The two axes correspond to whether entropy is computed after temporal aggregation of marginals or before it, and whether the temporal aggregation is discounted or stationary.}
\label{tab:infinite_horizon_efe}
\begin{tabularx}{\linewidth}{
    >{\centering\arraybackslash}m{2.3cm}
    >{\centering\arraybackslash}X
    >{\centering\arraybackslash}X
}
\toprule
& \textbf{Discounted} & \textbf{Stationary}\\
\textbf{time $\to$ state}
&
$\displaystyle
\begin{aligned}
d_\gamma^\pi &= (1-\gamma)\sum_{t=0}^{\infty}\gamma^t\mu_t^\pi\\[2pt]
\mathcal G_{\gamma}^{\mathrm{marg}}(\pi) &= \langle \ell, d_\gamma^\pi\rangle - \mathcal H(d^\pi)
\end{aligned}$
&
$\displaystyle
\begin{aligned}
d_\infty^\pi &= \lim_{T\to\infty}\frac1T\sum_{t=0}^{T-1}\mu_t^\pi\\[2pt]
\mathcal G_{\infty}^{\mathrm{marg}}(\pi) &= \langle \ell, d_\infty^\pi\rangle - \mathcal H(d_\infty^\pi)
\end{aligned}$
\\
\textbf{state $\to$ time}
&
$\displaystyle
\mathcal G_{\gamma}^{\mathrm{step}}(\pi) = (1-\gamma)\sum_{t=0}^{\infty}\gamma^t G_t(\pi)$
&
$\displaystyle
\mathcal G_{\infty}^{\mathrm{step}} = \lim_{T\to\infty}\frac1T\sum_{t=0}^{T-1} G_t$
\\
\bottomrule
\end{tabularx}
\end{table}

\paragraph{Discounted occupancy.}
Fix $\gamma\in(0,1)$ and a stationary $\pi$. The normalized discounted occupancy $d^\pi(x,a)=(1-\gamma)\sum_{t\ge0}\gamma^t\,\mu_t^\pi(x,a)$, with $\sum_{x,a}d^\pi(x,a)=1$, collapses the time-indexed flow constraints~\eqref{eq:flow}, summed against the weights $(1-\gamma)\gamma^t$, into the single Bellman-flow constraint
\begin{equation}\label{eq:disc-flow}
    \sum_a d^\pi(x',a)=(1-\gamma)\,q_0(x')+\gamma\sum_{x,a}p(x' | x,a)\,d^\pi(x,a),
    \qquad x'\in\cX.
\end{equation}
The discounted occupancy polytope $\mathcal{Q}_\gamma=\{d\in\Delta(\cX\times\cA):\eqref{eq:disc-flow}\}$ is again compact convex, and on full-support policies $\pi\mapsto d^\pi$ is a bijection with inverse $\pi(a | x)=d(x,a)/d(x)$ (as in the finite-horizon case, \cref{sec:convex}).

All of \cref{app:efe,app:md} applies with $\mu_t$ replaced by $d$ and the finite recursion of \cref{thm:dp} by the stationary soft-Bellman fixed point
\begin{align}
    Q^{(j+1)}(x,a)&=r^k(x,a)+\gamma\sum_{x'}p(x' | x,a)\,V^{(j)}(x'),\\
    V^{(j+1)}(x)&=\textstyle\sum_a\pi^{k+1}(a | x)\big[Q^{(j+1)}(x,a)-\eta_k^{-1}\log\tfrac{\pi^{k+1}(a | x)}{\pi^{k}(a | x)}\big],
\end{align}
with $r^k=-\nabla(\langle\ell,d\rangle+\Phi)(d^k)$, iterated in $j$ to a contraction tolerance inside each mirror step. The new policy is defined as in \cref{alg:mdaif}, only with the stationary $Q_\gamma$. This is Soft Actor Critic \cite{haarnoja2018soft} but with the previous policy $\pi_k$ instead of the uniform policy as the reference.

\paragraph{Average-reward limit.}
As $\gamma\to1$, the marginal convention $\cG^{\mathrm{marg}}$ yields the stationary occupancy $d^\pi_\infty(x,a) = \lim_{T\to\infty}\frac1T\sum_{t<T}\mu_t^\pi(x,a)$, which requires more than full support of the kernel: under a unichain assumption~\cite{puterman1994} on $p_\pi(x' | x)=\sum_a p(x' | x,a)\pi(a | x)$, every stationary $\pi$ induces a unique $d_\infty^\pi$ independent of $q_0$ and the limit exists. The per-step averages $\bar\ell(\pi)=\lim_T\frac1T\sum_t\E_{\mu_t^\pi}[\ell]$ and $\bar{\mathcal H}(\pi)=\lim_T\frac1T\sum_t\mathcal H(\mu_t^\pi)$ again satisfy $\bar{\mathcal H}(\pi)\ge\mathcal H(d_\infty^\pi)$, so aggregation and EFE computation do not commute. The convex-MDP structure and the mirror/natural-gradient solver survive in average-reward form for $\mathcal H(d_\infty^\pi)$, e.g.~\cite{adamczyk2025average}; the full AIF loop we leave to future work.

\section{Experimental details}\label{app:experiments}\label{app:pomdp}

All experiments are standard gridworlds. Every hyperparameter is listed in \cref{tab:hyper}. \emph{MD-AIF} (\cref{alg:mdaif}) starts from the uniform policy, uses a constant
step size, and floors each softmax
iterate with renormalization so the iterates stay interior, as the analysis requires. Its linearized reward is
$r_t^k(s,a)=\log\tilde p(s)-\log\rho_t^k(s)$. \emph{RL} is the same algorithm with the epistemic term $-\log\rho_t^k$ dropped, i.e.\ soft value iteration on the static preference reward.
\emph{EFE gradient descent} is Euclidean gradient descent on softmax logits $\theta_t\in\mathbb R^{S\times A}$,
differentiating the EFE exactly through the naive adjoint occupancy recursion (policy gradient on the gradient-to-go from the main text). The myopic AIF baselines~\cite{dacosta2020} roll the predicted belief forward from
each $s$ over lookahead $T_{\rm dc}$ and score each action sequence by the same EFE the other planners minimize, truncated to the lookahead.

\begin{table}[ht]
\centering
\caption{All hyperparameters. A dash means the setting does not apply.}
\label{tab:hyper}
\setlength{\tabcolsep}{5pt}
\begin{tabular}{@{}llcc@{}}
\toprule
 & & A ($5\times5$) & B ($10\times10$) \\
\midrule
\multicolumn{2}{@{}l}{\emph{Environment}} \\
\quad States / actions & $S$ / $A$ & $25$ / $4$ & $100$ / $4$ \\
\quad Slip probability & slip & --- & $0.05$ \\
\quad Emission ambiguity & ambig & --- & $0.05$ \\
\quad Ambiguous state & $s_{\rm amb}$ & --- & $55$ (centre) \\
\quad Log-preference & $\log\tilde p(s)$ & $-0.5\,d(s,s_{\rm
goal})$ & const. (unif.) \\
\midrule
\multicolumn{2}{@{}l}{\emph{Planning}} \\
\quad Horizon & $H$ & $12$ & $25$ \\
\quad Iterations & $K$ & $100$ & $120$ \\
\quad Step size & $\eta$ & $0.05$ & $0.05$ \\
\midrule
\multicolumn{2}{@{}l}{\emph{Closed loop}} \\
\quad Rounds & $R$ & --- & $20$ \\
\quad Episodes / round & $E$ & --- & $5$ \\
\midrule
\multicolumn{2}{@{}l}{\emph{Myopic AIF}~\cite{dacosta2020}} \\
\quad Lookahead & $T_{\rm dc}$ & --- & $3$ / $5$ \\
\quad MC samples & $N_{\rm mc}$ & --- & exact / $100$ \\
\midrule
\multicolumn{2}{@{}l}{\emph{Evaluation}} \\
\quad Seeds & & --- & $5$ \\
\quad Confidence interval & & --- & $95\%$ ($t$) \\
\bottomrule
\end{tabular}
\end{table}

\paragraph{A. Convergence ($5\times5$ grid).}
Deterministic grid, start at state~$0$ (top-left), goal at state~$24$
(bottom-right), with log-preferences $\log\tilde p(s)\propto-0.5\,d(s,s_{\rm
goal})$ for $d$ the Manhattan distance, normalized to the simplex. The
convergence plot is log--log in $\Gamma(\rho^k)-\Gamma^\star$, with $\Gamma^\star$
the minimum over all three methods and all iterates and the $O(1/k)$ reference
$C/k$, $C=\Gamma(\rho^0)-\Gamma^\star$. Occupancy snapshots show $\rho_t(s)$ at
$k=K$ for $t\in\{0,3,6,9,12\}$, MD-AIF and RL on a shared colour scale.

\paragraph{B. Model learning ($10\times10$ stochastic POMDP).}
Slippery grid, start state~$0$, goal state~$99$, and an emission
ambiguous at the single central state $s_{\rm amb}$, which spreads prob. mass over its neighbours while every other state emits exactly. Each round plans on the current $\hat p$ (assuming known emissions $E$),
deploys the policy, and refits $\hat p$ by Dirichlet--MAP.

The strictly positive start distribution ensures the positive state-marginals of the environment assumed in the simplified performativity argument of the main text. Over $5$ seeds (mean with a $95\%$ $t$-interval) the final mean
total-variation errors are MD-AIF $0.126\pm0.012$, EFE gradient descent $0.248\pm0.008$, myopic AIF $0.252\pm0.035$ ($T{=}5$) and $0.345\pm0.054$
($T{=}3$), reward-only RL $0.691\pm0.050$. MD-AIF's advantage over each baseline is significant (Welch's $t$-test, $p\le0.0003$).

\bibliographystyle{splncs04}
\IfFileExists{iwai/refs.bib}{\bibliography{iwai/refs}}{\bibliography{refs}}

\end{document}